\documentclass{llncs}
\usepackage{graphicx}
\usepackage{amsmath,amssymb,,subfig,tabularx,rotating,url} 
\usepackage{color}
\usepackage[width=122mm,left=12mm,paperwidth=146mm,height=193mm,top=12mm,paperheight=217mm]{geometry}
\begin{document}
\pagestyle{headings}
\mainmatter

\title{SSHMT: Semi-supervised Hierarchical Merge Tree for Electron Microscopy Image Segmentation} 

\titlerunning{SSHMT: Semi-supervised HMT for EM Image Segmentation}

\authorrunning{T.\ Liu, M.\ Zhang, M.\ Javanmardi, N.\ Ramesh, T.\ Tasdizen}



\author{Ting~Liu\inst{1}\and Miaomiao~Zhang\inst{2}\and Mehran~Javanmardi\inst{1}\and Nisha~Ramesh\inst{1}\and Tolga~Tasdizen\inst{1}}
\institute{Scientific Computing and Imaging Institute, University of Utah, USA\\\email{\{ting,mehran,nshramesh,tolga\}@sci.utah.edu}\\
  \and CSAIL, Massachusetts Institute of Technology, USA\\\email{miao86@mit.edu}}

\maketitle

\begin{abstract}
  Region-based methods have proven necessary for improving segmentation accuracy of neuronal structures in electron microscopy (EM) images. Most region-based segmentation methods use a scoring function to determine region merging. Such functions are usually learned with supervised algorithms that demand considerable ground truth data, which are costly to collect. We propose a semi-supervised approach that reduces this demand. Based on a merge tree structure, we develop a differentiable unsupervised loss term that enforces consistent predictions from the learned function. We then propose a Bayesian model that combines the supervised and the unsupervised information for probabilistic learning. The experimental results on three EM data sets demonstrate that by using a subset of only $3\%$ to $7\%$ of the entire ground truth data, our approach consistently performs close to the state-of-the-art supervised method with the full labeled data set, and significantly outperforms the supervised method with the same labeled subset.

\end{abstract}

\section{Introduction}
Connectomics researchers study structures of nervous systems to understand their function~\cite{sporns2005human}. Electron microscopy (EM) is the only modality capable of imaging substantial tissue volumes at sufficient resolution and has been used for the reconstruction of neural circuitry~\cite{famiglietti1991synaptic,briggman2011wiring,helmstaedter2013cellular}. The high resolution leads to image data sets at enormous scale, for which manual analysis is extremely laborious and can take decades to complete~\cite{briggman2006towards}. Therefore, reliable automatic connectome reconstruction from EM images, and as the first step, automatic segmentation of neuronal structures is crucial. However, due to the anisotropic nature, deformation, complex cellular structures and semantic ambiguity of the image data, automatic segmentation still remains challenging after years of active research.

Similar to the boundary detection/region segmentation pipeline for natural image segmentation~\cite{arbelaez2011contour,ren2013image,arbelaez2014multiscale,liu2016image}, most recent EM image segmentation methods use a membrane detection/cell segmentation pipeline. First, a membrane detector generates pixel-wise confidence maps of membrane predictions using local image cues~\cite{sommer2011ilastik,ciresan2012deep,seyedhosseini2013image}. Next, region-based methods are applied to transforming the membrane confidence maps into cell segments. It has been shown that region-based methods are necessary for improving the segmentation accuracy from membrane detections for EM images~\cite{arganda2015crowdsourcing}. A common approach to region-based segmentation is to transform a membrane confidence map into over-segmenting superpixels and use them as ``building blocks'' for final segmentation. To correctly combine superpixels, greedy region agglomeration based on certain boundary saliency has been shown to work~\cite{nunez2013machine}. Meanwhile, structures, such as loopy graphs~\cite{kaynig2015large,krasowski2015improving} or trees~\cite{liu2014modular,funke2015learning,uzunbas2016efficient}, are more often imposed to represent the region merging hierarchy and help transform the superpixel combination search into graph labeling problems. To this end, local~\cite{liu2014modular,krasowski2015improving} or structured~\cite{funke2015learning,uzunbas2016efficient} learning based methods are developed.

Most current region-based segmentation methods use a scoring function to determine how likely two adjacent regions should be combined. Such scoring functions are usually learned in a supervised manner that demands considerable amount of high-quality ground truth data. Obtaining such ground truth data, however, involves manual labeling of image pixels and is very labor intensive, especially given the large scale and complex structures of EM images. To alleviate this demand, Parag et al.\ recently propose an active learning framework~\cite{parag2014small,parag2015efficient} that starts with small sets of labeled samples and constantly measures the disagreement between a supervised classifier and a semi-supervised label propagation algorithm on unlabeled samples. Only the most disagreed samples are pushed to users for interactive labeling. The authors demonstrate that by using $15\%$ to $20\%$ of all labeled samples, the method can perform similar to the underlying fully supervised method with full training set. One disadvantage of this framework is that it does not directly explore the unsupervised information while searching for the optimal classification function. Also, retraining is required for the supervised algorithm at each iteration, which can be time consuming especially when more iterations with fewer samples per iteration are used to maximize the utilization of supervised information and minimize human effort. Moreover, repeated human interactions may lead to extra cost overhead in practice.

In this paper, we propose a semi-supervised learning framework for region-based neuron segmentation that seeks to reduce the demand for labeled data by exploiting the underlying correlation between unsupervised data samples. Based on the merge tree structure~\cite{liu2014modular,funke2015learning,uzunbas2016efficient}, we redefine the labeling constraint and formulate it into a differentiable loss function that can be effectively used to guide the unsupervised search in the function hypothesis space. We then develop a Bayesian model that incorporates both unsupervised and supervised information for probabilistic learning. The parameters that are essential to balancing the learning can be estimated from the data automatically. Our method works with very small amount of supervised data and requires no further human interaction. We show that by using only $3\%$ to $7\%$ of the labeled data, our method performs stably close to the state-of-the-art fully supervised algorithm with the entire supervised data set (Section~\ref{sec:res}). Also, our method can be conveniently adopted to replace the supervised algorithm in the active learning framework~\cite{parag2014small,parag2015efficient} and further improve the overall segmentation performance.

\section{Hierarchical Merge Tree}\label{sec:hmt}
Starting with an initial superpixel segmentation $S_o$ of an image, a merge tree $T=(\mathcal{V},\mathcal{E})$ is a graphical representation of superpixel merging order. Each node $v_i\in\mathcal{V}$ corresponds to an image region $s_i$. Each leaf node aligns with an initial superpixel in $S_o$. A non-leaf node corresponds to an image region combined by multiple superpixels, and the root node represents the whole image as a single region. An edge $e_{i,c}\in\mathcal{E}$ between $v_i$ and one of its child $v_c$ indicates $s_c\subset s_i$. Assuming only two regions are merged each time, we have $T$ as a full binary tree. A clique $p_i=(\{v_i,v_{c_1},v_{c_2}\},\{e_{i,c_1},e_{i,c_2}\})$ represents $s_i=s_{c_1}\cup s_{c_2}$. In this paper, we call clique $p_i$ is at node $v_i$. We call the cliques $p_{c_1}$ and $p_{c_2}$ at $v_{c_1}$ and $v_{c_2}$ the child cliques of $p_i$, and $p_i$ the parent clique of $p_{c_1}$ and $p_{c_2}$. If $v_i$ is a leaf node, $p_i=(\{v_i\},\varnothing)$ is called a leaf clique. We call $p_i$ a non-leaf/root/non-root clique if $v_i$ is a non-leaf/root/non-root node. An example merge tree, as shown in Fig.~\ref{fig:sub:toy_tree}, represents the merging of superpixels in Fig.~\ref{fig:sub:toy_segi}. The red box in Fig.~\ref{fig:sub:toy_tree} shows a non-leaf clique $p_7=(\{v_7,v_1,v_2\},\{e_{7,1},e_{7,2}\})$ as the child clique of $p_9=(\{v_9,v_7,v_3\},\{e_{9,7},e_{9,3}\})$. A common approach to building a merge tree is to greedily merge regions based on certain boundary saliency measurement in an iterative fashion~\cite{liu2014modular,funke2015learning,uzunbas2016efficient}.

\begin{figure}
\centering
\subfloat[\label{fig:sub:toy_segi}]{\includegraphics[width=0.25\textwidth]{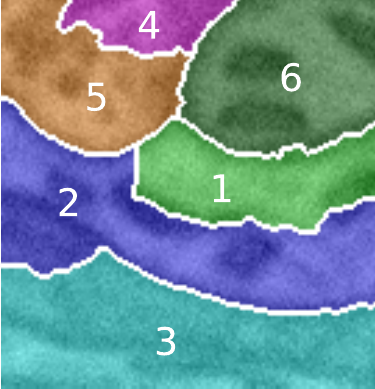}}\quad
\subfloat[\label{fig:sub:toy_seg}]{\includegraphics[width=0.25\textwidth]{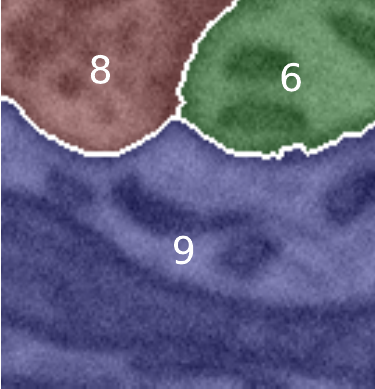}}\quad
\subfloat[\label{fig:sub:toy_tree}]{\includegraphics[width=0.44\textwidth]{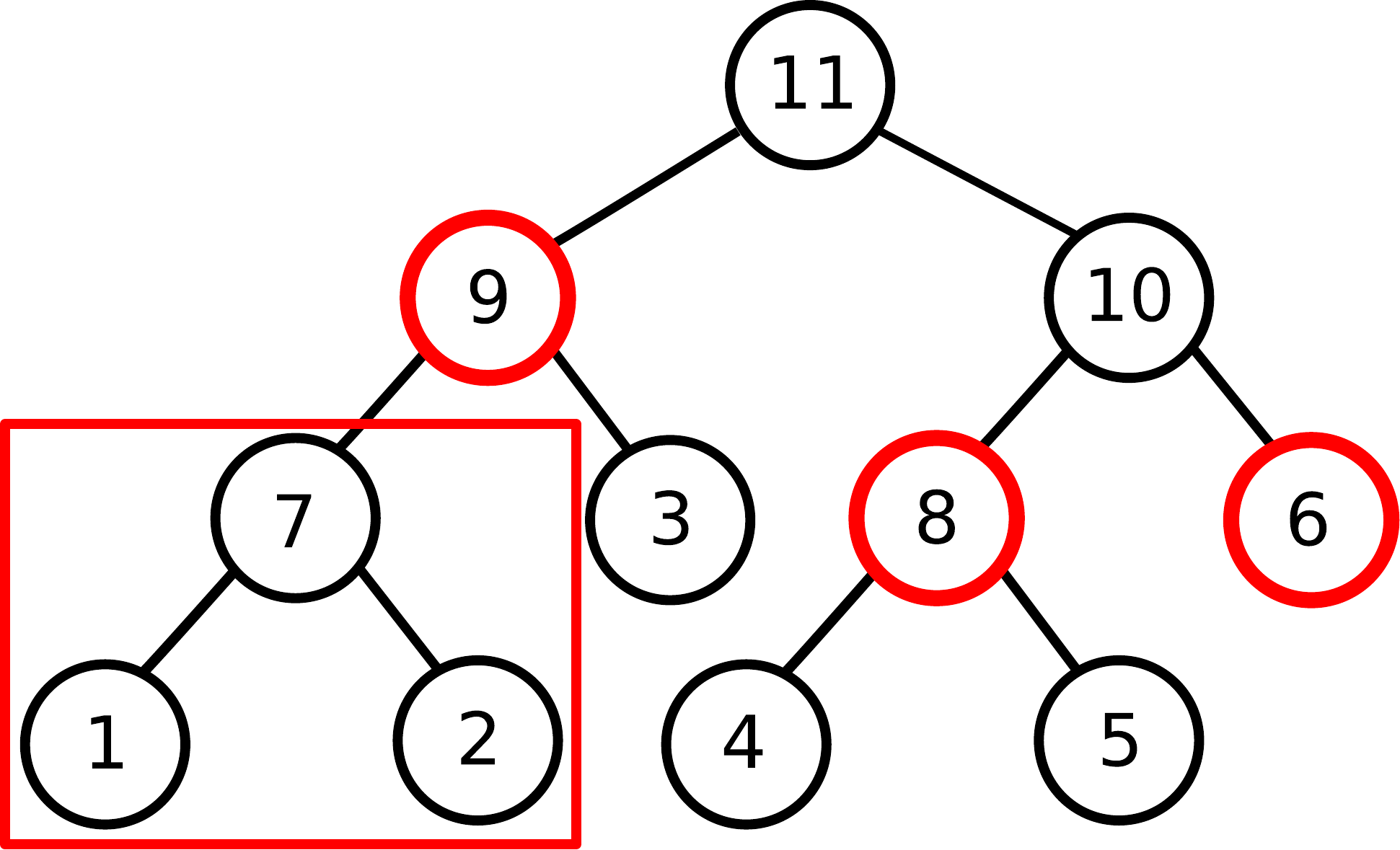}}
\caption{Example of (a) an initial superpixel segmentation, (b) a consistent final segmentation, and (c) the corresponding merge tree. The red nodes are selected ($z=1$) for the final segmentation, and the black nodes are not ($z=0$). The red box shows a clique.}\label{fig:toy}
\end{figure}

Given the merge tree, the problem of finding a final segmentation is equivalent to finding a complete label assignment $\mathbf{z}=\{z_i\}_{i=1}^{|\mathcal{V}|}$ for every node being a final segment ($z=1$) or not ($z=0$). Let $\rho(i)$ be a query function that returns the index of the parent node of $v_i$. The $k$-th ($k=1,\ldots d_i$) ancestor of $v_i$ is denoted as $\rho^k(i)$ with $d_i$ being the depth of $v_i$ in the tree, and $\rho^0(i)=i$. For every leaf-to-root path, we enforce the \emph{region consistency constraint} that requires $\sum_{k=0}^{d_i}z_{\rho^k(i)}=1$ for any leaf node $v_i$. As an example shown in Fig.~\ref{fig:sub:toy_tree}, the red nodes ($v_6$, $v_8$, and $v_9$) are labeled $z=1$ and correspond to the final segmentation in Fig.~\ref{fig:sub:toy_seg}. The rest black nodes are labeled $z=0$. Supervised algorithms are proposed to learn scoring functions in a local~\cite{liu2014modular,liu2016image} or a structured~\cite{funke2015learning,uzunbas2016efficient} fashion, followed by greedy~\cite{liu2014modular} or global~\cite{funke2015learning,liu2016image,uzunbas2016efficient} inference techniques for finding the optimal label assignment under the constraint. We refer to the local learning and greedy search inference framework in~\cite{liu2014modular} as the hierarchical merge tree (HMT) method and follow its settings in the rest of this paper, as it has been shown to achieve state-of-the-art results in the public challenges~\cite{arganda2015crowdsourcing,isbichallenge2013}.

A binary label $y_i$ is used to denote whether the region merging at clique $p_i$ occurs (``merge'', $y_i=1$) or not (``split'', $y_i=0$). For a leaf clique, $y=1$. At training time, $\mathbf{y}=\{y_i\}_{i=1}^{|\mathcal{V}|}$ is generated by comparing both the ``merge'' and ``split'' cases for non-leaf cliques against the ground truth segmentation under certain error metric (e.g.\ adapted Rand error~\cite{arganda2015crowdsourcing}). The one that causes the lower error is adopted. A binary classification function called the boundary classifier is trained with $(\mathbf{X},\mathbf{y})$, where $\mathbf{X}=\{\mathbf{x}_i\}_{i=1}^{|\mathcal{V}|}$ is a collection of feature vectors. Shape and image appearance features are commonly used.

At testing time, each non-leaf clique $p_i$ is assigned a likelihood score $P(y_i|\mathbf{x}_i)$ by the classifier. A potential for each node $v_i$ is defined as
\begin{equation}
u_i=P(y_i=1|\mathbf{x}_i)\cdot P(y_{\rho(i)}=0|\mathbf{x}_{\rho(i)}).\label{eq:node_potential}
\end{equation}
The greedy inference algorithm iteratively assigns $z=1$ to an unlabeled node with the highest potential and $z=0$ to its ancestor and descendant nodes until every node in the merge tree receives a label. The nodes with $z=1$ forms a final segmentation.

Note that HMT is not limited to segmenting images of any specific dimensionality. In practice, it has been successfully applied to both 2D~\cite{liu2014modular,arganda2015crowdsourcing} and 3D segmentation~\cite{isbichallenge2013} of EM images.

\section{SSHMT: Semi-supervised Hierarchical Merge Tree}\label{sec:sshmt}
The performance of HMT largely depends on accurate boundary predictions given fixed initial superpixels and tree structures. In this section, we propose a semi-supervised learning based HMT framework, named SSHMT, to learn accurate boundary classifiers with limited supervised data.

\subsection{Merge consistency constraint}
Following the HMT notation (Section~\ref{sec:hmt}), we first define the \emph{merge consistency constraint} for non-root cliques:
\begin{equation}
y_i\geq y_{\rho(i)},\forall i.\label{eq:mcc}
\end{equation}
Clearly, a set of consistent node labeling $\mathbf{z}$ can be transformed to a consistent $\mathbf{y}$ by assigning $y=1$ to the cliques at the nodes with $z=1$ and their descendant cliques and $y=0$ to the rest. A consistent $\mathbf{y}$ can be transformed to $\mathbf{z}$ by assigning $z=1$ to the nodes in $\{v_i\in\mathcal{V}|\forall i,\textrm{s.t.\ }y_i=1\wedge(v_i\textrm{ is the root}\vee y_{\rho(i)}=0)\}$ and $z=0$ to the rest, vice versa.

Define a clique path of length $L$ that starts at $p_i$ as an ordered set $\boldsymbol{\pi}^L_i=\{p_{\rho^l(i)}\}^{L-1}_{l=0}$. We then have
\begin{theorem}\label{theorem:cp_mono}
Any consistent label sequence $\mathbf{y}^L_i=\{y_{\rho^l(i)}\}_{l=0}^{L-1}$ for $\boldsymbol{\pi}^L_i$ under the merge consistency constraint is monotonically non-increasing.
\end{theorem}

\begin{proof}
Assume there exists a label sequence $\mathbf{y}^L_i$ subject to the merge consistency constraint that is not monotonically non-increasing. By definition, there must exist $k\geq0$, s.t.\ $y_{\rho^k(i)}<y_{\rho^{k+1}(i)}$. Let $j=\rho^k(i)$, then $\rho^{k+1}(i)=\rho(j)$, and thus $y_j<y_{\rho(j)}$. This violates the merge consistency constraint~\eqref{eq:mcc}, which contradicts the initial assumption that $\mathbf{y}^L_i$ is subject to the merge consistency constraint. Therefore, the initial assumption must be false, and all label sequences that are subject to the merge consistency constraint must be monotonically non-increasing.\qed
\end{proof}

Intuitively, Theorem~\ref{theorem:cp_mono} states that while moving up in a merge tree, once a split occurs, no merge shall occur again among the ancestor cliques in that path. As an example, a consistent label sequence for the clique path $\{p_7,p_9,p_{11}\}$ in Fig.~\ref{fig:sub:toy_tree} can only be $\{y_7,y_9,y_{11}\}=\{0,0,0\}$, $\{1,0,0\}$, $\{1,1,0\}$, or $\{1,1,1\}$. Any other label sequence, such as $\{1,0,1\}$, is not consistent. In contrast to the region consistency constraint, the merge consistency constraint is a local constraint that holds for the entire leaf-to-root clique paths as well as any of their subparts. This allows certain computations to be decomposed as shown later in Section~\ref{sec:res}.

Let $f_i$ be a predicate that denotes whether $y_i=1$. We can express the non-increasing monotonicity of any consistent label sequence for $\boldsymbol{\pi}^L_i$ in disjunctive normal form (DNF) as
\begin{equation}
  F^L_i=\bigvee_{j=0}^{L}\left(\bigwedge_{k=0}^{j-1}f_{\rho^k(i)}\wedge\bigwedge_{k=j}^{L-1}\neg f_{\rho^k(i)}\right),\label{eq:dnf}
\end{equation}
which always holds $true$ by Theorem~\ref{theorem:cp_mono}. We approximate $F^L_i$ with real-valued variables and operators by replacing $true$ with $1$, $false$ with $0$, and $f$ with real-valued $\tilde{f}$. A negation $\neg f$ is replaced by $1-\tilde{f}$; conjunctions are replaced by multiplications; disjunctions are transformed into negations of conjunctions using De Morgan's laws and then replaced. The real-valued DNF approximation is
\begin{equation}
\tilde{F}^L_i=1-\prod_{j=0}^L\left(1-\prod_{k=0}^{j-1}\tilde{f}_{\rho^k(i)}\cdot\prod_{k=j}^{L-1}\left(1-\tilde{f}_{\rho^k(i)}\right)\right),\label{eq:rdnf}
\end{equation}
which is valued $1$ for any consistent label assignments. Observing $\tilde{f}$ is exactly a binary boundary classifier in HMT, we further relax it to be a classification function that predicts $P(y=1|\mathbf{x})\in[0,1]$. The choice of $\tilde{f}$ can be arbitrary as long as it is (piecewise) differentiable (Section~\ref{sec:sub:sslearn}). In this paper, we use a logistic sigmoid function with a linear discriminant
\begin{equation}
  \tilde{f}(\mathbf{x};\boldsymbol{w})=\frac{1}{1+\exp(-\boldsymbol{w}^{\top}\mathbf{x})},\label{eq:logsig}
\end{equation}
which is parameterized by $\boldsymbol{w}$.

We would like to find an $\tilde{f}$ so that its predictions satisfy the DNF~\eqref{eq:rdnf} for any path in a merge tree. We will introduce the learning of such $\tilde{f}$ in a semi-supervised manner in Section~\ref{sec:sub:sslearn}.

\subsection{Bayesian semi-supervised learning}\label{sec:sub:sslearn}
To learn the boundary classification function $\tilde{f}$, we use both supervised and unsupervised data. Supervised data are the clique samples with labels that are generated from ground truth segmentations. Unsupervised samples are those we do not have labels for. They can be from the images that we do not have the ground truth for or wish to segment. We use $\mathbf{X}_s$ to denote the collection of supervised sample feature vectors and $\mathbf{y}_s$ for their true labels. $\mathbf{X}$ is the collection of all supervised and unsupervised samples.

Let $\boldsymbol{\tilde{f}}_{\boldsymbol{w}}=[\tilde{f}_{j_1},\ldots,\tilde{f}_{j_{N_s}}]^{\top}$ be the predictions about the supervised samples in $\mathbf{X}_s$, and $\boldsymbol{\tilde{F}}_{\boldsymbol{w}}=[\tilde{F}^L_{i_1},\ldots,\tilde{F}^L_{i_{N_u}}]^{\top}$ be the DNF values~\eqref{eq:rdnf} for all paths from $\mathbf{X}$. We are now ready to build a probabilistic model that includes a regularization prior,  an unsupervised likelihood, and a supervised likelihood.

The prior is an i.i.d.\ Gaussian $\mathcal{N}(0,1)$ that regularizes $\boldsymbol{w}$ to prevent overfitting. The unsupervised likelihood is an i.i.d.\ Gaussian $\mathcal{N}(0,\sigma_u)$ on the differences between each element of $\boldsymbol{\tilde{F}}_{\boldsymbol{w}}$ and $1$. It requires the predictions of $\tilde{f}$ to conform the merge consistency constraint for every path. Maximizing the unsupervised likelihood allows us to narrow down the potential solutions to a subset in the classifier hypothesis space without label information by exploring the sample feature representation commonality. The supervised likelihood is an i.i.d.\ Gaussian $\mathcal{N}(0,\sigma_s)$ on the prediction errors for supervised samples to enforce accurate predictions. It helps avoid consistent but trivial solutions of $\tilde{f}$, such as the ones that always predict $y=1$ or $y=0$, and guides the search towards the correct solution. The standard deviation parameters $\sigma_u$ and $\sigma_s$ control the contributions of the three terms. They can be preset to reflect our prior knowledge about the model distributions, tuned using a holdout set, or estimated from data.

By applying Bayes' rule, we have the posterior distribution of $\boldsymbol{w}$ as
\begin{equation}
  \begin{split}
    P(\boldsymbol{w}\,|\,\mathbf{X},\mathbf{X}_s,\mathbf{y}_s,\sigma_u,\sigma_s)\propto & \,P(\boldsymbol{w})\cdot P(\mathbf{1}\,|\,\mathbf{X},\boldsymbol{w},\sigma_u)\cdot P(\mathbf{y}_s\,|\,\mathbf{X}_s,\boldsymbol{w},\sigma_s)\\
    \propto & \,\exp\left(-\frac{\|\boldsymbol{w}\|_2^2}{2}\right)\\
    & \cdot\frac{1}{\left(\sqrt{2\pi}\sigma_u\right)^{N_u}}\exp\left(-\frac{\|\mathbf{1}-\boldsymbol{\tilde{F}}_{\boldsymbol{w}}\|_2^2}{2\sigma_u^2}\right)\\
    & \cdot\frac{1}{\left(\sqrt{2\pi}\sigma_s\right)^{N_s}}\exp\left(-\frac{\|\mathbf{y}_s-\boldsymbol{\tilde{f}}_{\boldsymbol{w}}\|_2^2}{2\sigma_s^2}\right),\label{eq:post}
  \end{split}
\end{equation}
where $N_u$ and $N_s$ are the number of elements in $\boldsymbol{\tilde{F}}_{\boldsymbol{w}}$ and $\boldsymbol{\tilde{f}}_{\boldsymbol{w}}$, respectively; $\mathbf{1}$ is a $N_u$-dimensional vector of ones.

\subsubsection{Inference}
We infer the model parameters $\boldsymbol{w}$, $\sigma_u$, and $\sigma_s$ using maximum a posteriori estimation. We effectively minimize the negative logarithm of the posterior
\begin{equation}
  \begin{split}
    J(\boldsymbol{w},\sigma_u,\sigma_s)= & \frac{1}{2}\|\boldsymbol{w}\|_2^2+\frac{1}{2\sigma_u^2}\|\mathbf{1}-\boldsymbol{\tilde{F}}_{\boldsymbol{w}}\|_2^2+N_u\log\sigma_u\\
    & +\frac{1}{2\sigma_s^2}\|\mathbf{y}_s-\boldsymbol{\tilde{f}}_{\boldsymbol{w}}\|_2^2+N_s\log\sigma_s.\label{eq:energy}
  \end{split}
\end{equation}

Observe that the DNF formula in~\eqref{eq:rdnf} is differentiable. With any (piecewise) differentiable choice of $\tilde{f}_{\boldsymbol{w}}$, we can minimize~\eqref{eq:energy} using (sub-) gradient descent. The gradient of~\eqref{eq:energy} with respect to the classifier parameter $\boldsymbol{w}$ is
\begin{equation}
  \nabla_{\boldsymbol{w}}J=\boldsymbol{w}^{\top}-\frac{1}{\sigma_u^2}\left(\mathbf{1}-\boldsymbol{\tilde{F}}_{\boldsymbol{w}}\right)^{\top}\nabla_{\boldsymbol{w}}\boldsymbol{\tilde{F}}_{\boldsymbol{w}}-\frac{1}{\sigma_s^2}\left(\mathbf{y}_s-\boldsymbol{\tilde{f}}_{\boldsymbol{w}}\right)^{\top}\nabla_{\boldsymbol{w}}\boldsymbol{\tilde{f}}_{\boldsymbol{w}},
\end{equation}

Since we choose $\tilde{f}$ to be a logistic sigmoid function with a linear discriminant~\eqref{eq:logsig}, the $j$-th ($j=1,\ldots,N_s$) row of $\nabla_{\boldsymbol{w}}\boldsymbol{\tilde{f}}_{\boldsymbol{w}}$ is
\begin{equation}
\nabla_{\boldsymbol{w}}\tilde{f}_j=\tilde{f}_j(1-\tilde{f}_j)\cdot\mathbf{x}_j^{\top}.\label{eq:dlogsig}
\end{equation}
where $\mathbf{x}_j$ is the $j$-th element in $\mathbf{X}_s$.

Define $g_j=\prod_{k=0}^{j-1}\tilde{f}_{\rho^k(i)}\cdot\prod_{k=j}^{L-1}(1-\tilde{f}_{\rho^k(i)})$, $j=0,\ldots,L$, we write~\eqref{eq:rdnf} as $\tilde{F}^L_i=1-\prod_{j=0}^L(1-g_j)$ as the $i$-th ($i=1,\ldots,N_u$) element of $\boldsymbol{\tilde{F}}_{\boldsymbol{w}}$. Then the $i$-th row of $\nabla_{\boldsymbol{w}}\boldsymbol{\tilde{F}}_{\boldsymbol{w}}$ is
\begin{equation}
\nabla_{\boldsymbol{w}}\tilde{F}^L_i=\sum_{j=0}^L\left(g_j\prod_{\substack{k=0\\k\neq j}}^L\left(1-g_k\right)\right)\left(\sum_{k=0}^{j-1}\frac{\nabla_{\boldsymbol{w}}\tilde{f}_{\rho^k(i)}}{\tilde{f}_{\rho^k(i)}}-\sum_{k=j}^{L-1}\frac{\nabla_{\boldsymbol{w}}\tilde{f}_{\rho^k(i)}}{1-\tilde{f}_{\rho^k(i)}}\right),\label{eq:ddnf}
\end{equation}
where $\nabla_{\boldsymbol{w}}\tilde{f}_{\rho^k(i)}$ can be computed using~\eqref{eq:dlogsig}.

We also alternately estimate $\sigma_u$ and $\sigma_s$ along with $\boldsymbol{w}$. Setting $\nabla_{\sigma_u}J=0$ and $\nabla_{\sigma_s}J=0$, we update $\sigma_u$ and $\sigma_s$ using the closed-form solutions
\begin{align}
\sigma_u= & \frac{\|\mathbf{1}-\boldsymbol{\tilde{F}}_{\boldsymbol{w}}\|_2}{\sqrt{N_u}}\\
\sigma_s= & \frac{\|\mathbf{y}_s-\boldsymbol{\tilde{f}}_{\boldsymbol{w}}\|_2}{\sqrt{N_s}}.
\end{align}

At testing time, we apply the learned $\tilde{f}$ to testing samples to predict their merging likelihood. Eventually, we compute the node potentials with~\eqref{eq:node_potential} and apply the greedy inference algorithm to acquire the final node label assignment (Section~\ref{sec:hmt}).

\section{Results}\label{sec:res}
We validate the proposed algorithm for 2D and 3D segmentation of neurons in three EM image data sets. For each data set, we apply SSHMT to the same segmentation tasks using different amounts of randomly selected subsets of ground truth data as the supervised sets.

\subsection{Data sets}

\subsubsection{Mouse neuropil data set}
\cite{deerinck2010enhancing} consists of $70$ 2D SBFSEM images of size $700\times700\times700$ at $10\times10\times50$ nm/pixel resolution. A random selection of $14$ images are considered as the whole supervised set, and the rest $56$ images are used for testing. We test our algorithm using $14$ ($100\%$), $7$ ($50\%$), $3$ ($21.42\%$), $2$ ($14.29\%$), $1$ ($7.143\%$), and half ($3.571\%$) ground truth image(s) as the supervised data. We use all the $70$ images as the unsupervised data for training. We target at 2D segmentation for this data set.

\subsubsection{Mouse cortex data set}
\cite{isbichallenge2013} is the original training set for the ISBI SNEMI3D Challenge~\cite{isbichallenge2013}. It is a $1024\times1024\times100$ SSSEM image stack at $6\times6\times30$ nm/pixel resolution. We use the first $1024\times1024\times50$ substack as the supervised set and the second $1024\times1024\times50$ substack for testing. There are $327$ ground truth neuron segments that are larger than $1000$ pixels in the supervised substack, which we consider as all the available supervised data. We test the performance of our algorithm by using $327$ ($100\%$), $163$ ($49.85\%$), $81$ ($24.77\%$), $40$ ($12.23\%$), $20$ ($6.116\%$), $10$ ($3.058\%$), and $5$ ($1.529\%$) true segments. Both the supervised and the testing substack are used for the unsupervised term. Due to the unavailability of the ground truth data, we did not experiment with the original testing image stack from the challenge. We target at 3D segmentation for this data set.

\subsubsection{\emph{Drosophila melanogaster} larval neuropil data set}
\cite{knott2008serial} is a $500\times500\times500$ FIBSEM image volume at $10\times10\times10$ nm/pixel resolution. We divide the whole volume evenly into eight $250\times250\times250$ subvolumes and do eight-fold cross validation using one subvolume each time as the supervised set and the whole volume as the testing data. Each subvolume has from $204$ to $260$ ground truth neuron segments that are larger than $100$ pixels. Following the setting in the mouse cortex data set experiment, we use subsets of $100\%$, $50\%$, $25\%$, $12.5\%$, $6.25\%$, and $3.125\%$ of all true neuron segments from the respective supervised subvolume in each fold of the cross validation as the supervised data to generate boundary classification labels. We use the entire volume to generate unsupervised samples. We target at 3D segmentation for this data set.

\subsection{Experiments}
We use fully trained Cascaded Hierarchical Models~\cite{seyedhosseini2013image} to generate membrane detection confidence maps and keep them fixed for the HMT and SSHMT experiments on each data set, respectively. To generate initial superpixels, we use the watershed algorithm~\cite{beucher1992morphological} over the membrane confidence maps. For the boundary classification, we use features including shape information (region size, perimeter, bounding box, boundary length, etc.) and image intensity statistics (mean, standard deviation, minimum, maximum, etc.) of region interior and boundary pixels from both the original EM images and membrane detection confidence maps.

We use the adapted Rand error metric~\cite{arganda2015crowdsourcing} to generate boundary classification labels using whole ground truth images (Section~\ref{sec:hmt}) for the 2D mouse neuropil data set. For the 3D mouse cortex and \emph{Drosophila melanogaster} larval neuropil data sets, we determine the labels using individual ground truth segments instead. We use this setting in order to match the actual process of analyzing EM images by neuroscientists. Details about label generation using individual ground truth segments are provided in Appendix~\ref{app}.

We can see in~\eqref{eq:rdnf} and~\eqref{eq:ddnf} that computing $\tilde{F}^L_i$ and its gradient involves multiplications of $L$ floating point numbers, which can cause underflow problems for leaf-to-root clique paths in a merge tree of even moderate height. To avoid this problem, we exploit the local property of the merge consistency constraint and compute $\tilde{F}^L_i$ for every path subpart of small length $L$. In this paper, we use $L=3$ for all experiments. For inference, we initialize $\boldsymbol{w}$ by running gradient descent on~\eqref{eq:energy} with only the supervised term and the regularizer before adding the unsupervised term for the whole optimization. We update $\sigma_u$ and $\sigma_s$ in between every $100$ gradient descent steps on $\boldsymbol{w}$.

We compare SSHMT with the fully supervised HMT~\cite{liu2014modular} as the baseline method. To make the comparison fair, we use the same logistic sigmoid function as the boundary classifier for both HMT and SSHMT. The fully supervised training uses the same Bayesian framework only without the unsupervised term in~\eqref{eq:energy} and alternately estimates $\sigma_s$ to balance the regularization term and the supervised term. All the hyperparameters are kept identical for HMT and SSHMT and fixed for all experiments. We use the adapted Rand error~\cite{arganda2015crowdsourcing} following the public EM image segmentation challenges~\cite{arganda2015crowdsourcing,isbichallenge2013}. Due to the randomness in the selection of supervised data, we repeat each experiment $50$ times, except in the cases that there are fewer possible combinations. We report the mean and standard deviation of errors for each set of repeats on the three data sets in Table~\ref{tab:res}. For the 2D mouse neuropil data set, we also threshold the membrane detection confidence maps at the optimal level, and the adapted Rand error is $0.2023$. Since the membrane detection confidence maps are generated in 2D, we do not measure the thresholding errors of the other 3D data sets. In addition, we report the results from using the globally optimal tree inference~\cite{liu2016image} in the supplementary materials for comparison.

\begin{table}
  \centering
  \caption{Means and standard deviations of the adapted Rand errors of HMT and SSHMT segmentations for the three EM data sets. The left table columns show the amount of used ground truth data, in terms of (a) the number of images, (b) the number of segments, and (c) the percentage of all segments. Bold numbers in the tables show the results of the higher accuracy under comparison. The figures on the right visualize the means (dashed lines) and the standard deviations (solid bars) of the errors of HMT (red) and SSHMT (blue) results for each data set.\\}\label{tab:res}
  \subfloat[Mouse neuropil\label{tab:sub:ncmir}]{
    \begin{tabular}[t]{cc}
      \begin{tabular}{|c|c|c|c|c|}
        \cline{2-5}
        \multicolumn{1}{c|}{} & \multicolumn{2}{c|}{HMT} & \multicolumn{2}{c|}{SSHMT}\\
        \hline
        $\#$GT & Mean & Std. & Mean & Std.\\
        \hline
        $14$ & $\mathbf{0.1135}$ & - & $0.1196$ & -\\
        $7$ & $0.1382$ & $0.03238$ & $\mathbf{0.1208}$ & $0.004033$\\
        $3$ & $0.1492$ & $0.04851$ & $\mathbf{0.1205}$ & $0.001383$\\
        $2$ & $0.1811$ & $0.07346$ & $\mathbf{0.1217}$ & $0.004116$\\
        $1$ & $0.2035$ & $0.1029$ & $\mathbf{0.1210}$ & $0.002206$\\
        $0.5$ & $0.2505$ & $0.1062$ & $\mathbf{0.1365}$ & $0.1079$\\
        \hline
        \multicolumn{5}{|l|}{Optimal thresholding: $0.2023$}\\
        \hline
      \end{tabular} &
      \begin{tabular}{c}
        \includegraphics[width=0.43\textwidth]{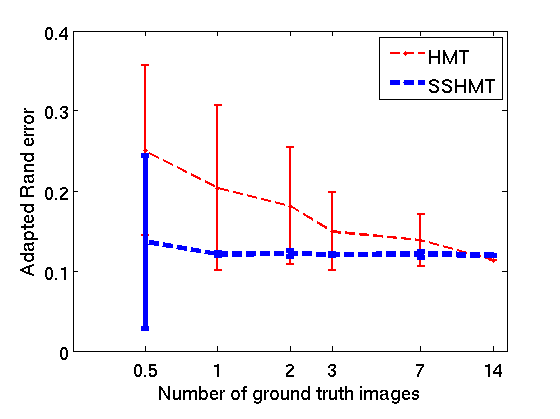}
      \end{tabular}
    \end{tabular}
  }\\
  \subfloat[Mouse cortex\label{tab:sub:ac4}]{
    \begin{tabular}[t]{cc}
      \begin{tabular}{|c|c|c|c|c|}
        \cline{2-5}
        \multicolumn{1}{c|}{} & \multicolumn{2}{c|}{HMT} & \multicolumn{2}{c|}{SSHMT}\\
        \hline
        $\#$GT & Mean & Std. & Mean & Std.\\
        \hline
        $327$ & $\mathbf{0.1101}$ & - & $0.1104$ & -\\
        $163$ & $0.1344$ & $0.03660$ & $\mathbf{0.1189}$ & $0.01506$\\
        $81$ & $0.1583$ & $0.06909$ & $\mathbf{0.1215}$ & $0.01661$\\
        $40$ & $0.1844$ & $0.1019$ & $\mathbf{0.1198}$ & $0.01690$\\
        $20$ & $0.2205$ & $0.1226$ & $\mathbf{0.1238}$ & $0.01466$\\
        $10$ & $0.2503$ & $0.1561$ & $\mathbf{0.1219}$ & $0.01273$\\
        $5$ & $0.4389$ & $0.2769$ & $\mathbf{0.2008}$ & $0.2285$\\
        \hline
      \end{tabular} &
      \begin{tabular}{c}
        \includegraphics[width=0.43\textwidth]{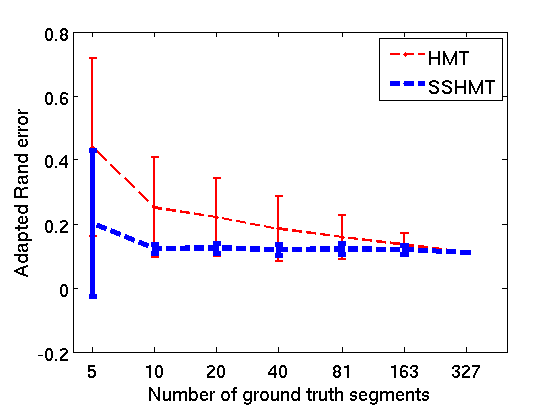}
      \end{tabular}
    \end{tabular}
  }\\
  \subfloat[\emph{Drosophila melanogaster} larval neuropil\label{tab:sub:fibsem}]{
    \begin{tabular}[t]{cc}
      \begin{tabular}{|c|c|c|c|c|}
        \cline{2-5}
        \multicolumn{1}{c|}{} & \multicolumn{2}{c|}{HMT} & \multicolumn{2}{c|}{SSHMT}\\
        \hline
        $\%$GT & Mean & Std. & Mean & Std.\\
        \hline
        $100\%$ & $0.06044$ & - & $\mathbf{0.05504}$ & -\\
        $50\%$ & $0.09004$ & $0.04476$ & $\mathbf{0.05602}$ & $0.005550$\\
        $25\%$ & $0.1240$ & $0.07491$ & $\mathbf{0.05803}$ & $0.007703$\\
        $12.5\%$ & $0.1418$ & $0.1055$ & $\mathbf{0.05835}$ & $0.007797$\\
        $6.25\%$ & $0.1748$ & $0.1389$ & $\mathbf{0.05756}$ & $0.008933$\\
        $3.125\%$ & $0.2017$ & $0.1871$ & $\mathbf{0.06213}$ & $0.03660$\\
        \hline
      \end{tabular} &
      \begin{tabular}{c}
        \includegraphics[width=0.43\textwidth]{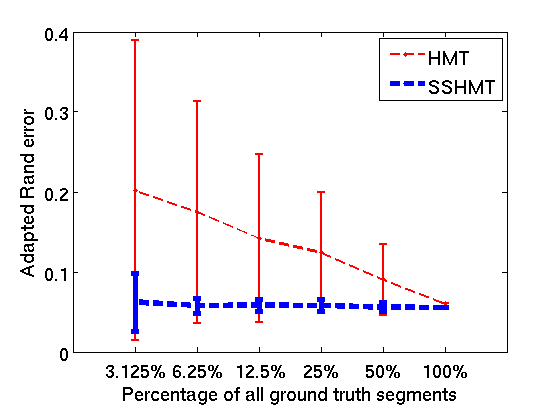}
      \end{tabular}
    \end{tabular}
  }
\end{table}

Examples of 2D segmentation testing results from the mouse neuropil data set using fully supervised HMT and SSHMT with $1$ ($7.143\%$) ground truth image as supervised data are shown in Fig.~\ref{fig:viz:ncmir}. Examples of 3D individual neuron segmentation testing results from the \emph{Drosophila melanogaster} larval neuropil data set using fully supervised HMT and SSHMT with $12$ ($6.25\%$) true neuron segments as supervised data are shown in Fig.~\ref{fig:viz:fibsem}.

\begin{figure}
\centering
\begin{tabular}{cccc}
\includegraphics[width=0.23\textwidth]{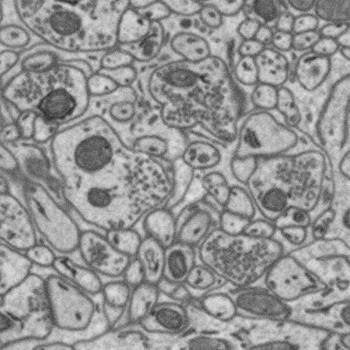} &
\includegraphics[width=0.23\textwidth]{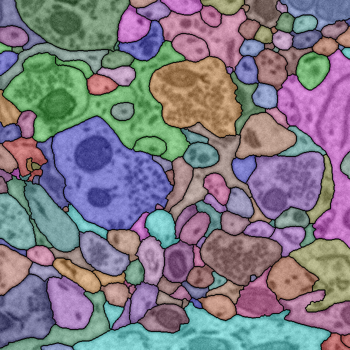} &
\includegraphics[width=0.23\textwidth]{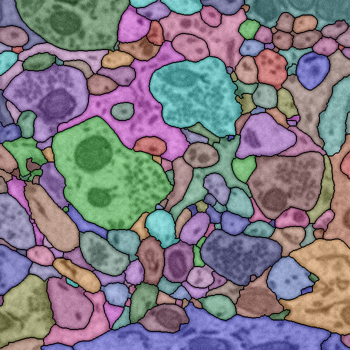} &
\includegraphics[width=0.23\textwidth]{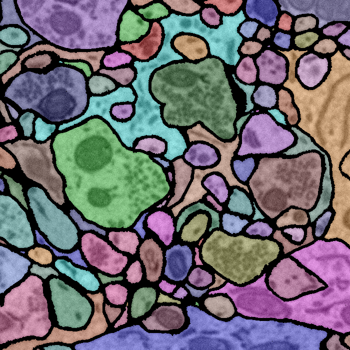} \\
\includegraphics[width=0.23\textwidth]{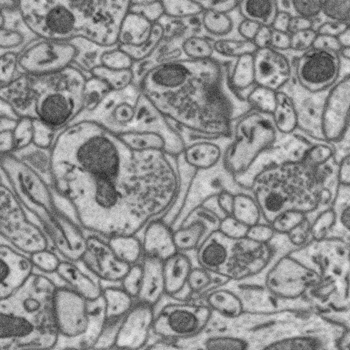} &
\includegraphics[width=0.23\textwidth]{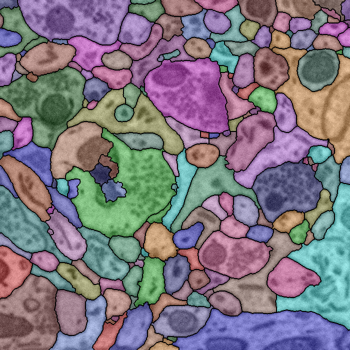} &
\includegraphics[width=0.23\textwidth]{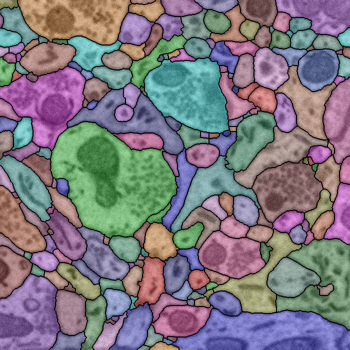} &
\includegraphics[width=0.23\textwidth]{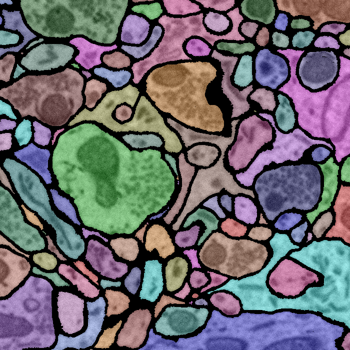} \\
\includegraphics[width=0.23\textwidth]{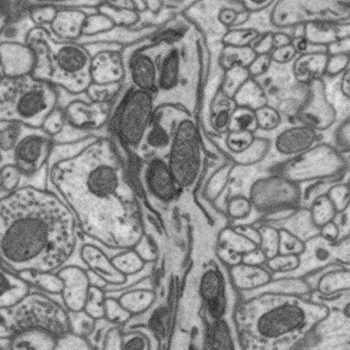} &
\includegraphics[width=0.23\textwidth]{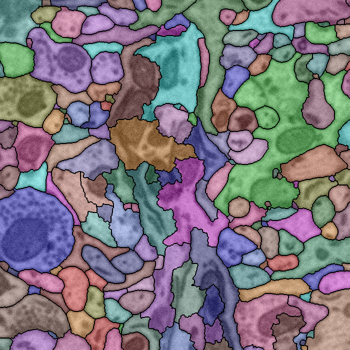} &
\includegraphics[width=0.23\textwidth]{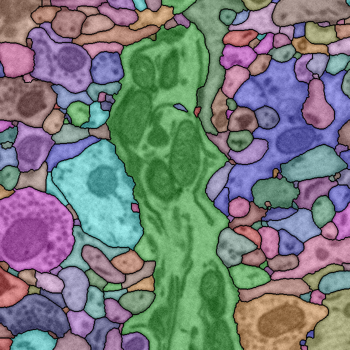} &
\includegraphics[width=0.23\textwidth]{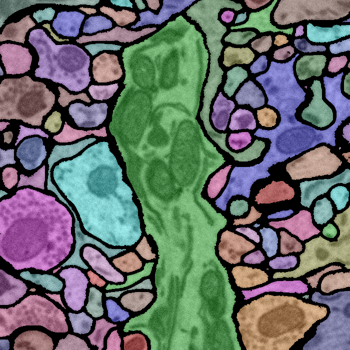} \\
\includegraphics[width=0.23\textwidth]{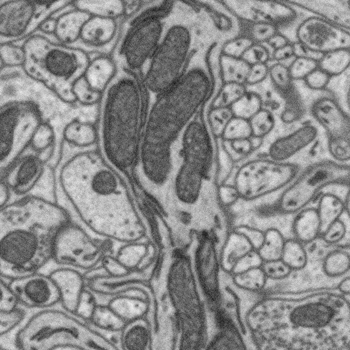} &
\includegraphics[width=0.23\textwidth]{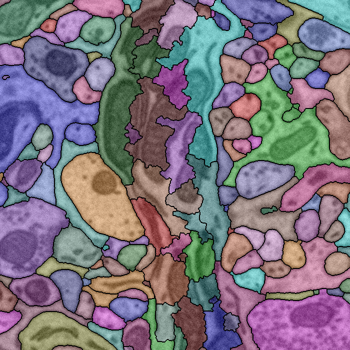} &
\includegraphics[width=0.23\textwidth]{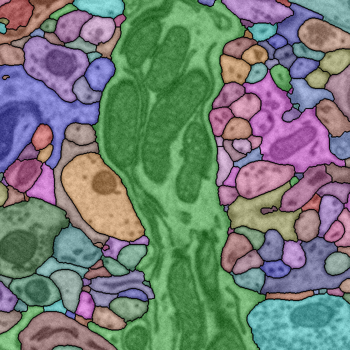} &
\includegraphics[width=0.23\textwidth]{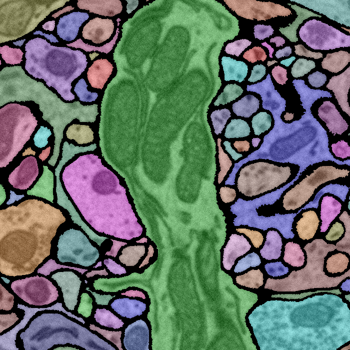} \\
\includegraphics[width=0.23\textwidth]{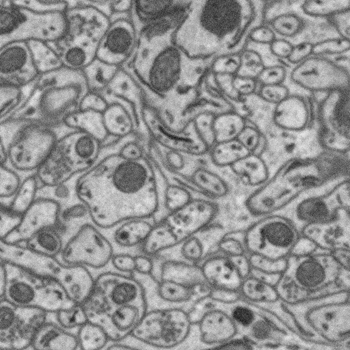} &
\includegraphics[width=0.23\textwidth]{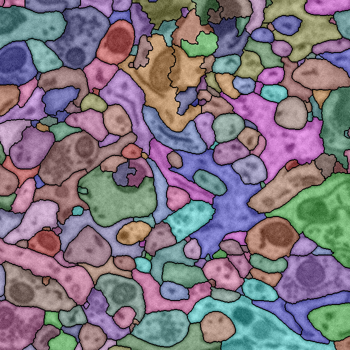} &
\includegraphics[width=0.23\textwidth]{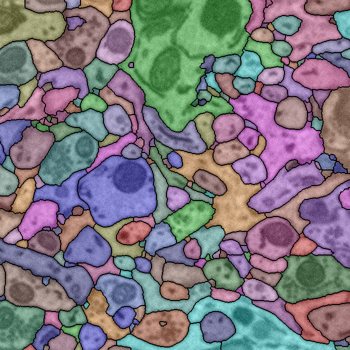} &
\includegraphics[width=0.23\textwidth]{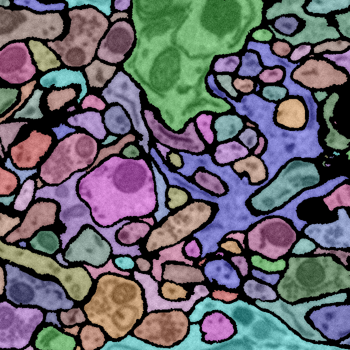} \\
(a) Original & (b) HMT & (c) SSHMT & (d) Ground truth\\
\end{tabular}
\caption{Examples of the 2D segmentation testing results for the mouse neuropil data set, including (a) original EM images, (b) HMT and (c) SSHMT results using $1$ ground truth image as supervised data, and (d) the corresponding ground truth images. Different colors indicate different individual segments.}\label{fig:viz:ncmir}
\end{figure}

\begin{figure}
\centering
\begin{tabular}{ccc}
\includegraphics[width=0.27\textwidth]{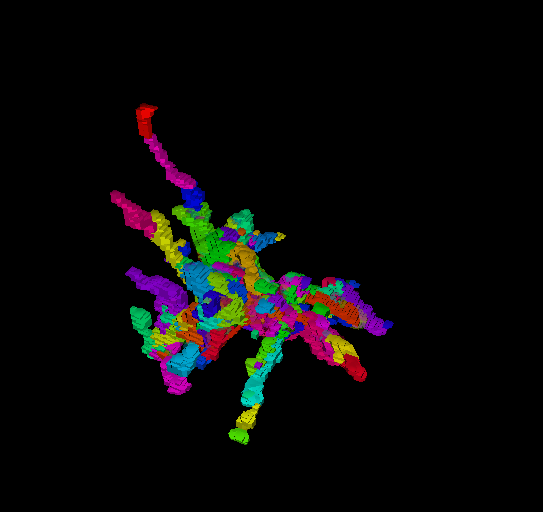} &
\includegraphics[width=0.27\textwidth]{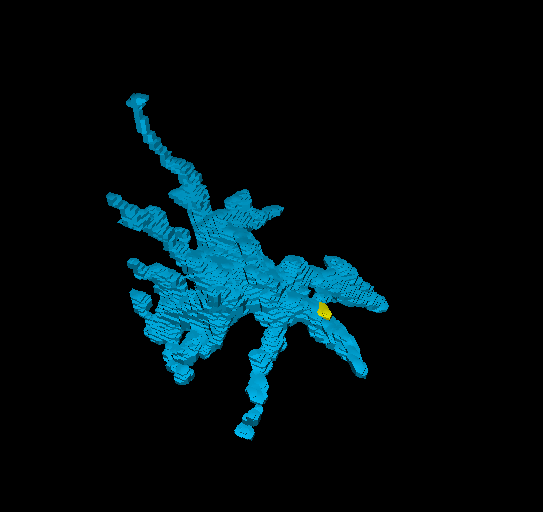} &
\includegraphics[width=0.27\textwidth]{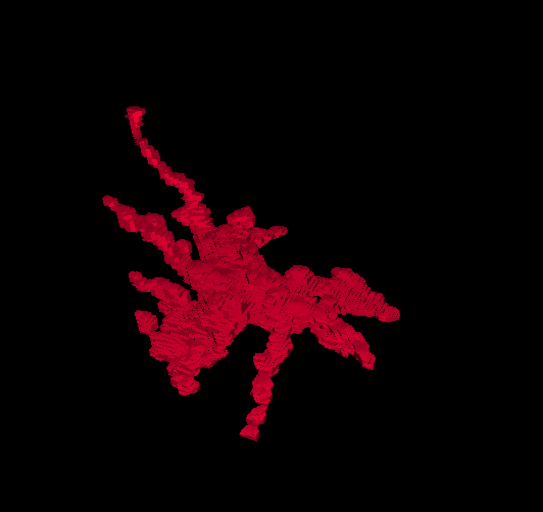}\\
\includegraphics[width=0.27\textwidth]{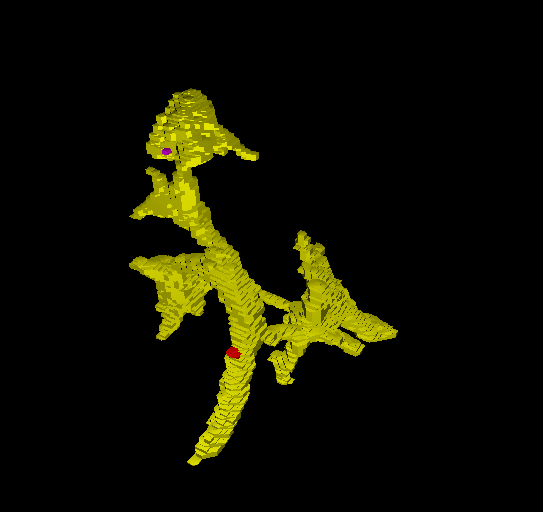} &
\includegraphics[width=0.27\textwidth]{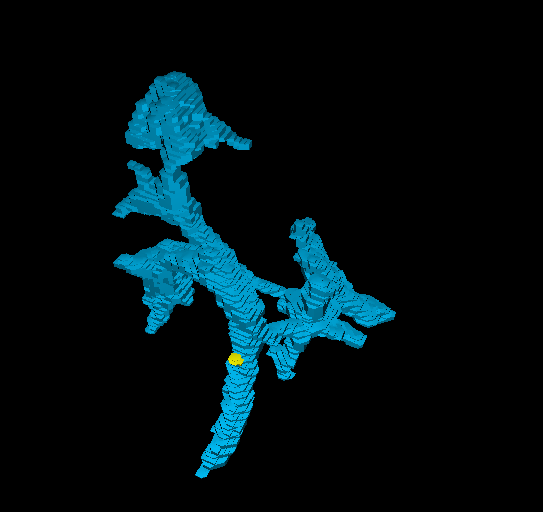} &
\includegraphics[width=0.27\textwidth]{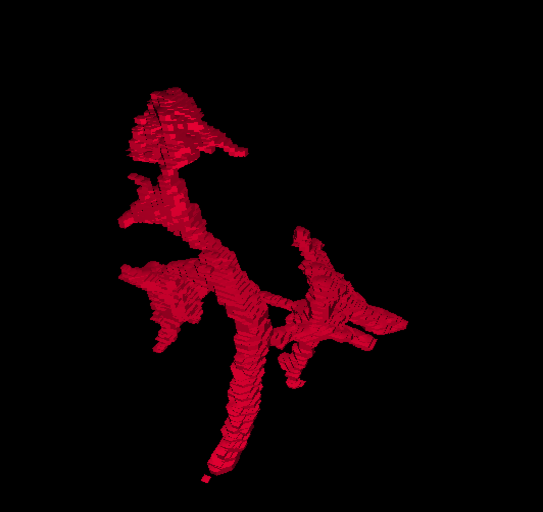}\\
\includegraphics[width=0.27\textwidth]{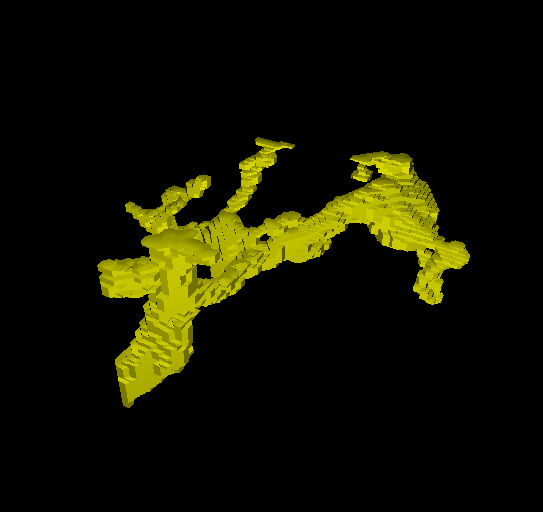} &
\includegraphics[width=0.27\textwidth]{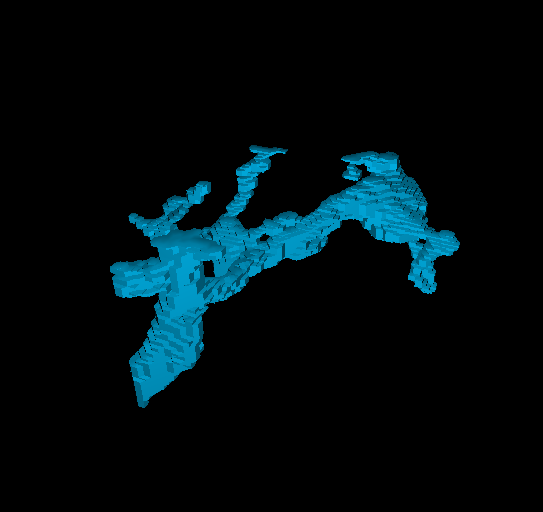} &
\includegraphics[width=0.27\textwidth]{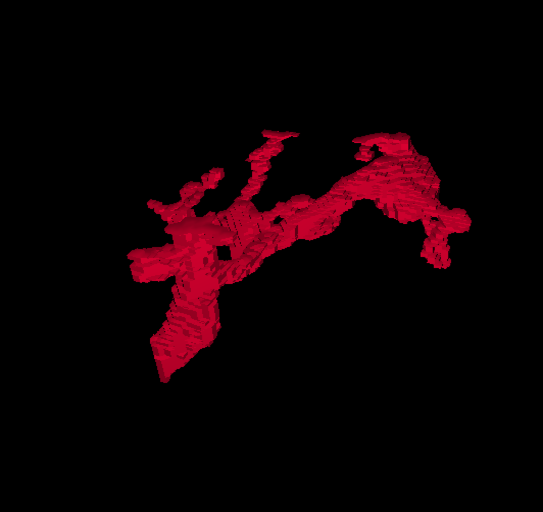}\\
\includegraphics[width=0.27\textwidth]{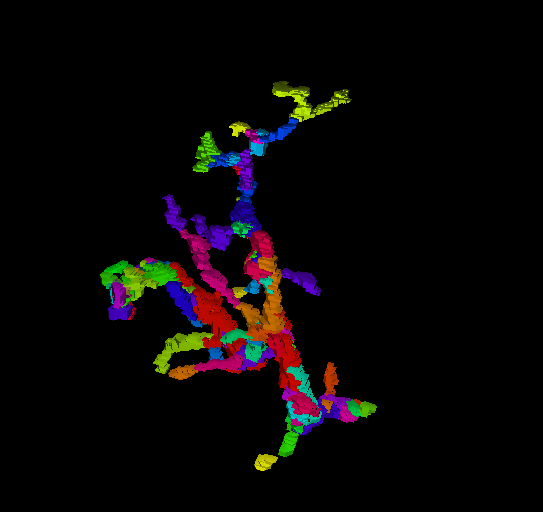} &
\includegraphics[width=0.27\textwidth]{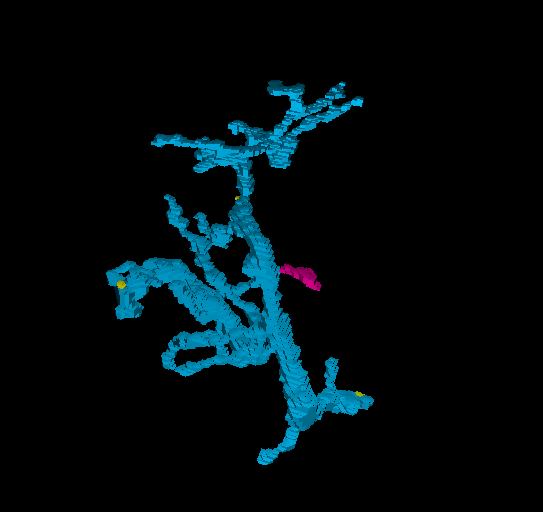} &
\includegraphics[width=0.27\textwidth]{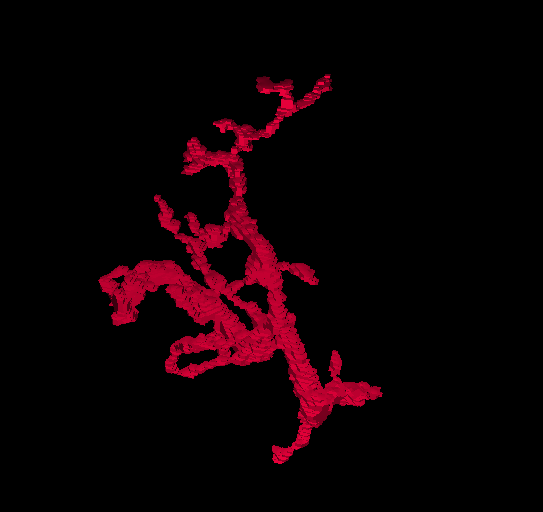}\\
\includegraphics[width=0.27\textwidth]{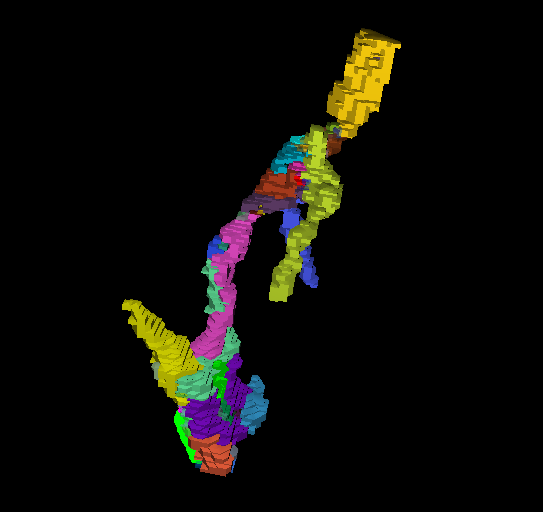} &
\includegraphics[width=0.27\textwidth]{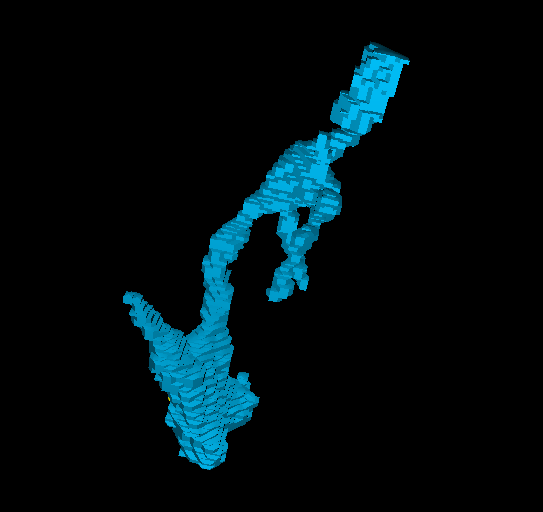} &
\includegraphics[width=0.27\textwidth]{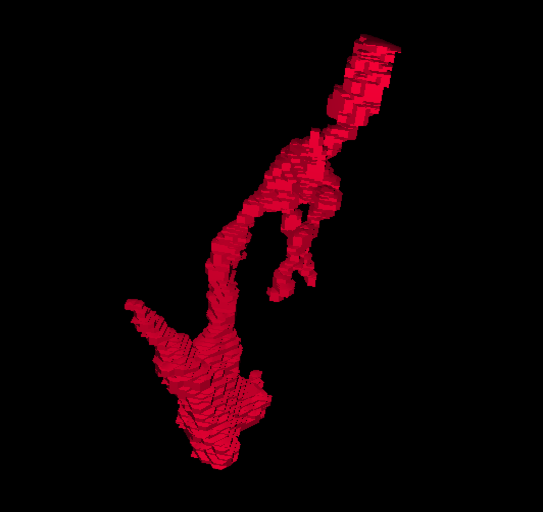}\\
(a) HMT & (b) SSHMT & (c) Ground truth\\
\end{tabular}
\caption{Examples of individual neurons from the 3D segmentation testing results for the \emph{Drosophila melanogaster} larval neuropil data set, including (a) HMT and (b) SSHMT results using $12$ ($6.25\%$) 3D ground truth segments as supervised data, and (c) the corresponding ground truth segments. Different colors indicate different individual segments. The 3D visualizations are generated using Fiji~\cite{schindelin2012fiji}.}\label{fig:viz:fibsem}
\end{figure}

From Table~\ref{tab:res}, we can see that with abundant supervised data, the performance of SSHMT is similar to HMT in terms of segmentation accuracy, and both of them significantly improve from optimally thresholding (Table~\ref{tab:sub:ncmir}). When the amount of supervised data becomes smaller, SSHMT significantly outperforms the fully supervised method with the accuracy close to the HMT results using the full supervised sets. Moreover, the introduction of the unsupervised term stabilizes the learning of the classification function and results in much more consistent segmentation performance, even when only very limited ($3\%$ to $7\%$) label data are available. Increases in errors and large variations are observed in the SSHMT results when the supervised data become too scarce. This is because the few supervised samples are incapable of providing sufficient guidance to balance the unsupervised term, and the boundary classifiers are biased to give trivial predictions.

Fig.~\ref{fig:viz:ncmir} shows that SSHMT is capable of fixing both over- and under-segmentation errors that occur in the HMT results. Fig.~\ref{fig:viz:fibsem} also shows that SSHMT can fix over-segmentation errors and generate highly accurate neuron segmentations. Note that in our experiments, we always randomly select the supervised data subsets. For realistic uses, we expect supervised samples of better representativeness to be provided with expertise and the performance of SSHMT to be further improved.

We also conducted an experiment with the mouse neuropil data set in which we use only $1$ ground truth image to train the membrane detector, HMT, and SSHMT to test a fully semi-supervised EM segmentation pipeline. We repeat $14$ times for every ground truth image in the supervised set. The optimal thresholding gives adapted Rand error $0.3603\pm 0.06827$. The error of the HMT results is $0.2904\pm 0.09303$, and the error of the SSHMT results is $0.2373\pm 0.06827$. Despite the increase of error, which is mainly due to the fully supervised nature of the membrane detection algorithm, SSHMT again improves the region accuracy from optimal thresholding and has a clear advantage over HMT.

We have open-sourced our code at \url{https://github.com/tingliu/glia}. It takes approximately 80 seconds for our SSHMT implementation to train and test on the whole mouse neuropil data set using $50$ $2.5$ GHz Intel Xeon CPUs and about $150$ MB memory.

\section{Conclusion}
In this paper, we proposed a semi-supervised method that can consistently learn boundary classifiers with very limited amount of supervised data for region-based image segmentation. This dramatically reduces the high demands for ground truth data by fully supervised algorithms. We applied our method to neuron segmentation in EM images from three data sets and demonstrated that by using only a small amount of ground truth data, our method performed close to the state-of-the-art fully supervised method with full labeled data sets. In our future work, we will explore the integration of the proposed constraint based unsupervised loss in structural learning settings to further exploit the structured information for learning the boundary classification function. Also, we may replace the current logistic sigmoid function with more complex classifiers and combine our method with active learning frameworks to improve segmentation accuracy.

\subsubsection{Acknowledgment} This work was supported by NSF IIS-1149299 and NIH 1R01NS075314-01. We thank the National Center for Microscopy and Imaging Research at the University of California, San Diego, for providing the mouse neuropil data set. We also thank Mehdi Sajjadi at the University of Utah for the constructive discussions.

\appendix
\section{Appendix: Generating Boundary Classification Labels Using Individual Ground Truth Segments}\label{app}
Assume we only have individual annotated image segments instead of entire image volumes as ground truth. Given a merge tree, we generate the best-effort ground truth classification labels for a subset of cliques as follows:
\begin{enumerate}
\item For every region represented by a tree node, compute the Jaccard indices of this region against all the annotated ground truth segments. Use the highest Jaccard index of each node as its eligible score.
\item Mark every node in the tree as ``eligible'' if its eligible score is above certain threshold ($0.75$ in practice) or ``ineligible'' otherwise.
\item Iteratively select a currently ``eligible'' node with the highest eligible score; mark it and its ancestors and descendants as ``ineligible'', until every node is ``ineligible''. This procedure generates a set of selected nodes.
\item For every selected node, label the cliques at itself and its descendants as $y=1$ (``merge'') and the cliques at its ancestors as $y=0$ (``split'').
\end{enumerate}

Eventually, the clique samples that receive merge/split labels are considered as the supervised data.

\bibliographystyle{splncs}
\bibliography{refs-arxiv}

\newpage
\title{SSHMT: Semi-supervised Hierarchical Merge Tree for Electron Microscopy Image Segmentation}
\subtitle{Supplementary Materials}
\author{Ting~Liu\inst{1}\and Miaomiao~Zhang\inst{2}\and Mehran~Javanmardi\inst{1}\and Nisha~Ramesh\inst{1}\and Tolga~Tasdizen\inst{1}}
\institute{Scientific Computing and Imaging Institute, University of Utah, USA\\\email{\{ting,mehran,nshramesh,tolga\}@sci.utah.edu}\\
  \and CSAIL, Massachusetts Institute of Technology, USA\\\email{miao86@mit.edu}}

\maketitle

Under the same experiment setting as in Section 4.2, we report the results from using~[9] in Table~S1, which considers the merge tree structure as a constrained conditional model (CCM) and computes globally optimal solutions based on supervised learning for inference.

\begin{table}
  \centering
  \caption*{Table~S1: Means and standard deviations of the adapted Rand errors of CCM segmentations for (a) the mouse neuropil and (b) the \emph{Drosophila melanogaster} larval neuropil data sets. The left table columns show the amount of used ground truth data, in terms of (a) the number of images and (b) the percentage of all segments.}\label{tab:res}
  \subfloat[Mouse neuropil\label{tab:sub:ncmir}]{
    \begin{tabular}{|c|c|c|}
      \cline{2-3}
      \multicolumn{1}{c|}{} & \multicolumn{2}{c|}{CCM}\\
      \hline
      $\#$GT & Mean & Std.\\
      \hline
      $14$ & $0.1166$ & - \\
      $7$ & $0.1238$ & $0.01548$\\
      $3$ & $0.1278$ & $0.02957$\\
      $2$ & $0.1412$ & $0.03900$\\
      $1$ & $0.1465$ & $0.03738$\\
      $0.5$ & $0.1900$ & $0.08586$\\
      \hline
    \end{tabular}
  }\qquad
  \subfloat[\emph{Drosophila}\label{tab:sub:fibsem}]{
    \begin{tabular}{|c|c|c|}
      \cline{2-3}
      \multicolumn{1}{c|}{} & \multicolumn{2}{c|}{CCM}\\
      \hline
      $\#$GT & Mean & Std.\\
      \hline
      $100\%$ & $0.05812$ & - \\
      $50\%$ & $0.06612$ & $0.01868$\\
      $25\%$ & $0.08067$ & $0.05281$\\
      $12.5\%$ & $0.08262$ & $0.04714$\\
      $6.25\%$ & $0.08186$ & $0.04463$\\
      $3.125\%$ & $0.09784$ & $0.09035$\\
      \hline
    \end{tabular}
  }
\end{table}

Comparing Table~S1 with Table~1 in Section 4.2, we can see that even though the supervised CCM improves from HMT, our SSHMT still consistently outperforms it with a clear margin. Also, the globally optimal inference algorithm in CCM can be used in combination with the proposed semi-supervised learning framework conveniently. We experienced a data loss of the mouse cortex dataset due to power outage, so we did not experiment on this dataset, but we expect similar results.
\end{document}